\documentclass[pmlr,twocolumn,a4paper,10pt]{jmlr}
\usepackage{isipta2023}
\usepackage[american]{babel}
\usepackage{booktabs} 
\usepackage{tikz} 
\usepackage{dirtytalk}
\usepackage{hyperref}
\usepackage{bm}
\usepackage{float}
\usepackage[]{algorithm, algpseudocode} 

\DeclareMathOperator*{\argmax}{arg\,max}

\jmlrworkshop{Preprint}

\title[Selecting Pseudo-Labeled Data Reliably]{
  In all Likelihood\textit{s}: How to Reliably Select Pseudo-Labeled Data for Self-Training in Semi-Supervised Learning
}

\author{
  \Name{Julian Rodemann}\Email{julian@stat-uni-muenchen.de}\\
  \Name{Christoph Jansen}\Email{christoph.jansen@stat.uni-muenchen.de}\\
  \Name{Georg Schollmeyer}\Email{georg.schollmeyer@stat.uni-muenchen.de}\\
    \Name{Thomas Augustin}\Email{thomas.augustin@stat.uni-muenchen.de}\\
  \addr Department of Statistics, Ludwig-Maximilians-Universität (LMU) Munich
}

\begin{document}

\maketitle

\begin{abstract}

Self-training is a simple yet effective method within semi-supervised learning. The idea is to iteratively enhance training data by adding pseudo-labeled data. Its generalization performance heavily depends on the selection of these pseudo-labeled data (PLS). 
In this paper, we aim at rendering PLS more robust towards the involved modeling assumptions. To this end, we propose to select pseudo-labeled data that maximize a multi-objective utility function. The latter is constructed to account for different sources of uncertainty, three of which we discuss in more detail: model selection, accumulation of errors and covariate shift. In the absence of second-order information on such uncertainties, we furthermore consider the generic approach of the generalized Bayesian $\alpha$-cut updating rule for credal sets. As a practical proof of concept, we spotlight the application of three of our robust extensions on simulated and real-world data. Results suggest that in particular robustness w.r.t. model choice can lead to substantial accuracy gains.\footnote{\textbf{Open Science:} Implementations of the proposed methods and
reproducible scripts for the experimental analysis are available at:\\ \href{https://github.com/rodemann/reliable-pls}{www.github.com/rodemann/reliable-pls}.
%
}





\end{abstract}


\section{Introduction}\label{sec:intro}

Labels for observations are burdensome to obtain in a myriad of applied learning tasks ranging from image classification~\cite{soleymani2022deep} over financial econometrics~\cite{shin2013prediction} to genomics~\cite{gunduz2021self}. This scarcity of labeled data has given rise to the paradigm of \textit{semi-supervised learning} (SSL). Within SSL, \textit{self-training} (also called pseudo-labeling) is often considered the most straight-forward approach~\cite{shi2018transductive, lee2013pseudo, mcclosky2006effective}. Self-training follows the general rationale of iteratively assigning pseudo-labels to unlabeled data according to the model's predictions. More precisely, the idea is to predict classes of unlabeled data by means of a model trained on labeled data and include some of the predictions as pseudo-labeled data in the training data, before predicting on the remaining unlabeled data again. This process requires a criterion (called confidence measure) for \textit{pseudo-label selection} (PLS), that is, the selection of pseudo-labeled instances to be added to the training data. 
What most of these confidence measures have in common is the fact of stemming uniquely and exclusively from one sole model. The paper at hand aims at a selection of pseudo-labeled data with regard to a variety of (fitted) models, rendering PLS robust with regard to model imprecision. The latter can have multiple sources. Section~\ref{sec:robust-pls} discusses how to deal with three of them in detail: model selection, accumulation of errors and covariate shift. In case such sources are not identifiable, we propose a generic robust approach to PLS in section~\ref{sec:alpha-cuts}, building on the rich literature on credal sets and generalized Bayesian inference. The remainder of this section discusses related work and introduces semi-supervised learning formally, leaning on~\cite{rodemann2023-bpls}. 
The paper concludes with a brief real world application and some concluding remarks in chapters~\ref{sec:application}~and~\ref{sec:discussion}.

\label{sec:background}
\subsection{Semi-supervised Learning}
\label{sec:ssl}
The vast majority of SSL methods is concerned with classification tasks~\cite{van2020survey, chapelle2006semi}. Loosely leaning on~\cite{triguero2015self}, we formalize SSL as follows. Consider labeled data 
\begin{equation}
\label{eq:unlabeled-data}
    \mathcal{D}=\left\{\left(x_{i}, y_{i}\right)\right\}_{i=1}^{n} \in \left(\mathcal{X} \times \mathcal{Y}\right)^{n}
\end{equation} and unlabeled data 
\begin{equation}
    \mathcal{U}=\left\{\left(x_{i}, \mathcal{Y}\right)_i\right\}_{i=n+1}^{m} \in \left(\mathcal{X} \times 2^\mathcal{Y}\right)^{m-n}
\end{equation} from the same data generation process, where $\mathcal{X}$ is the feature space and $\mathcal{Y}$ is the categorical target space. The aim of SSL is to learn a predictive classification function $\hat y_{\theta}(x)$ parameterized by $\theta$ utilizing both $\mathcal{D}$ and $\mathcal{U}$. 
The objective can be twofold~\cite{triguero2015self}. On the one hand, one simply aims at labeling $\mathcal{U}$ (transductive learning). On the other hand, and more commonly, both $\mathcal{D}$ and $\mathcal{U}$ can be used to learn a prediction function to predict any unseen test data (inductive learning) in a more accurate way than only relying on $\mathcal{D}$ as in classical supervised learning.  

\subsection{Self-training}
\label{sec:self-training}

 According to~\cite{pise2008survey} and~\cite{van2020survey}, SSL can be broadly categorized into self-training and co-training. We will focus on the former, whose general idea is commonly described as fitting a model on $\mathcal{D}$ by empirical risk minimization and then exploiting this model's predictions to label $\mathcal{U}$. Typically, those instances from $\mathcal{U}$ are added whose predictions are most confident according to some confidence measure. The predicted probability (probability score) is among the most popular ones~\cite{triguero2015self}. Besides, the predictions' variance as well as a linear combination of variance and probability score are used~\cite{rizve2020defense}. Regarding the inclusion of pseudo-labeled data from $\mathcal{U}$ to $\mathcal{D}$,~\cite{triguero2015self} and~\cite{kostopoulos2018semi} distinguish between incremental, batch-wise, and amending mechanisms. The incremental approaches label instances one-by-one in a sequential fashion, whereas batch-wise and amending techniques allow for adding of multiple data points or removal of data, respectively. Moreover,~\cite{triguero2015self} differentiate self-training methods into single- and multi-classifier ones, depending on the number of learned classifiers $\hat{y}(x)$ used while labeling. If multiple classifiers are used, they can either be based on the same model class or a variety of models. This is known as single- versus multi-learning, see~\cite{triguero2015self} for instance. Combining and aggregating the predictions and confidence measures of multiple classifiers can be done in various ways. This is slightly related to our proposed model-robust PLS, see sections~\ref{sec:robust-pls} and~\ref{sec:alpha-cuts}. The difference is of course that we \textit{select} pseudo-labeled data in the light of multiple models, while multi-learning deploys multiple models for \textit{predicting} pseudo-labels. 
 
 Additionally, self-training algorithms may have different stopping criteria~\cite{triguero2015self}. A naive option is to label and add the entire set $\mathcal{U}$. Alternatively, one could stop when $\hat y(x)$, the predictive classifier, no longer changes due to $\mathcal{U}$, leaving the remaining data in $\mathcal{U}$ unlabeled. In this paper, we propose both incremental and batch-wise approaches that can be used with any stopping criterion for the purpose of inductive learning.

\subsection{Superset Learning}
\label{sec:superset}

The notion of superset learning is a generalization of semi-supervised learning. Instead of completely unlabeled (i.e. fully ambiguos) data $\mathcal{U}=\left\{\left(x_{i}, \mathcal{Y}\right)_{i}\right\}_{i=1}^{m} \in\left(\mathcal{X} \times 2^\mathcal{Y}\right)^{m}$, superset learning considers $\left\{\left(x_{i}, Y_{i}\right)\right\}_{i=1}^{m} \in\left(\mathcal{X} \times 2^\mathcal{Y}\right)^{m}$, where $Y_i \subseteq \mathcal{Y}$. In this context, $Y_i$ is regarded a superset of a \say{true} underlying singleton $y_i$, thus the name. There exist optimistic as well as pessimistic variants of superset learning and approaches to balance these extreme cases~\cite{hullermeier2014learning, hullermeier2015superset, hullermeier2019learning, rodemann2022supersetlearning}. The general idea is to find a singleton representation (often called instantiation) of the supersets that corresponds to the most predictive (optimistic) or least predictive (pessimistic) model when trained and evaluated on it. In the optimistic case, this can be achieved by minimizing an optimistic version of the empirical risk, the generalized empirical risk: $\frac{1}{n} \sum_{i=1}^n L^*(\hat{y}_i, Y_i) =   \frac{1}{n} \sum_{i=1}^n \min_{y \in Y_i} L(\hat{y}_i , y)$ with $L^*$ the \textit{optimistic superset} or \textit{infimum loss}~\cite{cabannnes2020structured}. 

\subsection{Robust Semi-Supervised Learning}

The robustness of SSL and in particular of self-training has been widely discussed.~\cite{aminian2022information} propose an information-theoretic approach to pseudo-label prediction which is resistant to covariate shift.~\cite{vandewalle2013predictive} worked to make self-training more robust to modeling assumptions by allowing model selection through the deviance information criterion. Coming close to our use of credal sets in section~\ref{sec:alpha-cuts},~\cite{lienen2021credal, lienen2022conformal} suggest identifying pseudo-labels as sets of probability distributions (\say{credal self-supervised learning}).  Inspired by consistency regularization~\cite{Berthelot2019, Sohn2020, Zhang2021c}, superset learning~\cite{hullermeier2014learning, hullermeier2015superset, hullermeier2019learning, rodemann2022supersetlearning} and distributional alignment~\cite{remixmatch}, \say{credal self-supervised learning} aims at decreasing the reliance on a single distributional assumption. Our work follows the same rationale, while being conceptually different:~\cite{lienen2021credal, lienen2022conformal} start by imprecisiation of the training data by means of soft labels through data augmentation, thus obtaining set-valued predictions. In this paper, we exploit the expressiveness of credal sets only in the selection phase. Generally, there appears to be a large body of research on robustifying \textit{predictions} in SSL by means of Bayesian techniques~\cite{gordon2020combining, ng2018bayesian, adams2009archipelago}, weighted likelihood~\cite{sokolovska2008asymptotics}, conditional likelihood~\cite{grandvalet2004semi}, and joint mixture likelihood~\cite{amini2002semi}. On the other hand, there is only limited (Bayesian) or hardly any (likelihood-based) work regarding robust versions of Bayesian or likelihood-based \textit{selection} of pseudo-labels, which is the very idea of the paper at hand. The authors of~\cite{li2020pseudo} quantify the uncertainties of pseudo-labels by mixtures of predictive distributions of a neural net, utilizing MC dropout, thus simulating a Bayesian setup without explicitly considering the posterior predictive. More recently,~\cite{patel2022seq} proposed PLS with respect to the entropy of the pseudo-labels' posterior predictive distribution. 

\cite{rodemann2023-bpls} tackle the problem of pseudo-label selection (PLS) in semi-supervised learning from the viewpoint of decision theory, proposing Bayes optimal pseudo-label selection (BPLS). The idea is to make PLS more robust towards the initial fit by marginalizing over the parameters’ posterior instead of considering the predictive distribution of a single best parameter vector. While this allows for selecting pseudo-labeled data in light of more than one fit of a given model, BPLS is still restricted to the assumed (type of) model and the distributional assumptions that come with it. This is the very starting point for several robust extensions of BPLS, that will be presented in the main part of this paper, namely sections~\ref{sec:robust-pls} and~\ref{sec:alpha-cuts}. To begin with, we introduce the conditional view on PLS in section~\ref{sec:cond-pls}. This allows our understanding of (B)PLS as decision problems, as explained in section~\ref{sec:dec-problem}.

\section{Pseudo-Label Selection}
\label{sec:PLS}

\subsection{Conditional Pseudo-Label Selection}
\label{sec:cond-pls}

As in standard self-training, we start by fitting a parametric model $M$ with unknown parameter vector $\theta \in \Theta$ on labeled data $
    \mathcal{D}=\left\{\left(x_{i}, y_{i}\right)\right\}_{i=1}^{n} $. In this work, we assume $\Theta$ to be compact and denote $\dim(\Theta) = q$. Note that Bayesian inference smoothly integrates in this setup, since we might state a prior function over $\Theta$ for a given parametric model $M$ as $\pi(\theta \mid M)$. We aim at learning the conditional distribution of $p(\bm y \mid \bm x)$ through $\theta$ from observing features $\bm{x} = (x_1, \dots, x_n), x_i \in \mathcal{X}$, and classes $\bm{y} = (y_1, \dots, y_n), y_i \in \mathcal{Y}$ in $\mathcal{D}$. As touched upon in sections~\ref{sec:ssl} and~\ref{sec:self-training}, we start by estimating $\hat{\theta} \in \Theta$ from $M$ through the labeled data $\mathcal{D}$, predict on unlabeled data $\mathcal{U}$ and select those predicted (pseudo-labeled) data points that we are most confident in according to some selection criterion and add them to $\mathcal{D}$.
    
    Most importantly, throughout this paper, we do not deal with predicting unknown labels of  $\mathcal{U}=\left\{\left(x_{i}, \mathcal{Y}\right)_i\right\}_{i=1}^{m}$ by the fitted model on $\mathcal{D}$. Rather, we are primarily concerned with the problem of \textit{selecting} from those already predicted. That is, we identify each element in $\mathcal{U}=\left\{\left(x_{i}, \mathcal{Y}\right)_i\right\}_{i=n+1}^{m} \in \left(\mathcal{X} \times 2^\mathcal{Y}\right)^{m-n}$
with its corresponding prediction $\left\{\left(x_{i}, \hat{y}\right)\right\}_i$, obtaining 
$
    \hat{\mathcal{U}} = \left\{\left(x_{i}, \hat{y} \right)\right\}_{i=n+1}^{m} \in \left(\mathcal{X} \times \mathcal{Y}\right)^{m-n}.
$ However, we will stick with $\mathcal{U}$ in the following to emphasize that our reasoning holds for any functional $\hat y(x)$. This is not to say that we completely abstain from any specifications of the prediction method, see remark~\ref{remark:redundancy psl}, where we will rely on maximum-likelihood estimation.

\subsection{PLS as Decision Problem}
\label{sec:dec-problem}

Following~\cite{rodemann2023-bpls}, we formalize pseudo-label selection as a canonical decision problem with likelihood utility and thus lay the groundwork for several robust extensions of classical decision criteria.

\begin{definition}[Canonical Decision Problem]
\label{def:dec-probl}    
Define $(\mathbb{A}, \Theta, u(\cdot))$ as decision-theoretic triple with an action space $\mathbb{A}$, an unknown set of states of nature $\Theta$ and a utility function $u : \mathbb{A} \times \Theta \to \mathbb{R}$. 

\end{definition}

Throughout this section, we are concerned with the decision of selecting pseudo-labeled data, where an action corresponds to the selection of an instance from the unlabeled data $\mathbb{A}_{\mathcal{U}} = \{(z, \mathcal{Y}) \mid \exists \, i \in \{n+1, \dots, m\} : (z, \mathcal{Y}) = (x_i, \mathcal{Y})_i \in \mathcal{U}    \}$, i.e., instances as actions $ \mathbb{A}_{\mathcal{U}} \ni a = (z, \mathcal{Y})$. This is in stark contrast to statistical decision theory, where estimators instead of data are to be selected. The decision for an action is guided by a utility function. Closely following~\cite{rodemann2023-bpls} and loosely inspired by~\cite{cattaneo2007statistical, cattaneo2013likelihood}, we proceed by defining the utility of a selected data point $(z, \mathcal{Y}) = (x_i, \mathcal{Y})_i$ as the plausibility of being generated jointly with $\mathcal{D}$ by a model $M$ with states (parameters) $\theta \in \Theta$ if we include it with its predicted pseudo-label $\hat y(z) = \hat y(x_i) = \hat y_i \in \mathcal{Y}$ in $\mathcal{D} \cup (x_i, \hat{y}_i)$, see definition~\ref{def:pseud-lik}.

\begin{definition}[Pseudo-Label Likelihood as Utility]
\label{def:pseud-lik}
Given $\mathcal{D}$ and the prediction functional $\hat y: \mathcal{X} \to \mathcal{Y}$, we define the following utility function\begin{align*}
  u \colon \mathbb{A}_{\mathcal{U}} \times \Theta &\to \mathbb{R}\\
  ((z, \mathcal{Y}), \theta) &\mapsto u((z, \mathcal{Y}), \theta) = p(\mathcal{D} \cup (z, \hat{y}(z))\mid \theta, M),
  \end{align*}
  which is said to be the pseudo-label likelihood. In the following, for ease of exposition, we will write $\ell(i) := p(\mathcal{D} := p(i \mid \theta, M)  \cup (x_i, \hat{y}(x_i))\mid \theta, M)$ for the pseudo-label likelihood. 
\end{definition}

Based on this embedding of PLS in decision theory, classical decision criteria such as max-max or the Bayes criterion can be derived.~\cite[chapter 2.2]{rodemann2023-bpls} shows that the former corresponds to optimistic superset learning~\cite{hullermeier2015superset} and the latter to the posterior predictive of data to be pseudo-labeled $p(\mathcal{D} \cup (x_i, \hat{y}_i)\mid \mathcal{D}, M)$, subsequently called pseudo posterior predictive (PPP). 
The \textit{max-max-criterion} is defined by $ \Phi_{m} \colon \mathbb{A}_{\mathcal{U}} \to \mathbb{R}; a \mapsto \max_\theta u(a,\theta) $. Each element of $ \argmax \Phi_m$ is then called a \textit{max-max-action}. 
The \textit{Bayes-criterion} given $\pi$ is defined by $\Phi_{\pi} \colon \mathbb{A}_{\mathcal{U}} \to \mathbb{R}; \, a \mapsto  \mathbb{E}_\pi(u(a,\theta)) $. Each element of $ \argmax\Phi_{\pi}$ is then called \textit{Bayes-action}.

\section{Robust PLS: In All Likelihood\textit{s}}
\label{sec:robust-pls}
Within common approaches to self-training in SSL, it might well be possible to generalize and robustify models used for \textit{predicting} pseudo-labels. In the following, however, we aim at robust \textit{selection} of pseudo-labeled data, see section~\ref{sec:cond-pls}. To this end, we will modify the generic utility function (definition~\ref{def:pseud-lik}) and the respective Bayes criterion~\cite[chapter 2.2]{rodemann2023-bpls} to account for three frequent sources of uncertainty and imprecision: model selection, accumulation of errors and covariate shift. Instead of relying on likelihood utilities from models that are assumed to be correct \say{in all likelihood}, we suggest relying on all likelihood\textbf{s} from multiple models.

\subsection{Model Selection: Reversing Occam's razor}
\label{sec:model-sel}

An obvious and ubiquitous source of imprecision is the model choice: The likelihood under which distributional assumption (and corresponding model) should be taken into account? So far, we have defined the pseudo-label likelihood as the one under the model $M$ that we have used for predicting pseudo-labels. Albeit, this is far from necessary. As discussed above, our conditional approach to choosing pseudo-labeled data renders this selection completely orthogonal to predicting pseudo-labels. Instead of defining the utility function (see~\ref{def:pseud-lik}) as the likelihood of observing the pseudo-labeled data under the assumptions of model $M$, we might as well consider $\tilde M$ or a weighted sum of likelihoods under several models. 
In what follows, we start with the generic case of any finite number of different models that can be parameterized in a meaningful way and work our way through nested models, ending with nested generalized linear models, and discuss how to account for their specifications in PLS. 

\subsubsection{Generic Case} \label{GC}

Start by considering any $M_1, \dots, M_K$, $K < \infty$, different parametric models specified on respective parameter spaces $\Theta_1, \dots, \Theta_K$. Denote by $\tilde \Theta = \times_{k = 1}^K \Theta_k$ their Cartesian product and by $f_k: \tilde \Theta \to \Theta_k$, $k \in \{1, \dots, K\}$ the projections from the Cartesian product to each $\Theta_k$. We can easily extend the pseudo-label likelihood utility (definition~\ref{def:pseud-lik}) to account for several models, inducing a multiobjective decision problem.

\begin{definition}[Multi-Model Likelihood Utility]
\label{def:multi-model-util}
As in definition~\ref{def:pseud-lik} consider $\mathcal{D}$ and pseudo-labels $\hat{y} \in \mathcal{Y}$ from $\hat y: \mathcal{X} \to \mathcal{Y}$ as given. The $K$-dimensional utility function
\begin{align*}
  u \colon \mathbb{A}_{\mathcal{U}} \times \tilde \Theta &\to \mathbb{R}^K\\
  ((x_i, \mathcal{Y})_i, \theta) &\mapsto (\ell(i,1), \dots, \ell(i,K))'
  \end{align*}  
shall be called multi-model likelihood. We write $\ell(i,k) = p(i \mid f_k(\theta), M_k) = p(\mathcal{D} \cup (z, \hat{y}(z))\mid f_k(\theta), M_k)$ with $\theta_k \in \Theta_k$ for brevity. Let $K$ again denote the number of models under consideration. 
\end{definition}


For the optimization of such a multiobjective utility considered in definition~\ref{def:multi-model-util} one is faced with a multicriteria decision problem. For such decision problems there are lots of solution strategies. One modern way to deal with a multidimensional utility function  was recently proposed in~\cite{jsa2022}. The idea is -- utilizing that each \textit{single} dimension considered is perfectly cardinal -- to embed the image of the utility function into a \textit{preference system} $\mathcal{A}$, i.e. into a specific order-theoretic structure allowing for modeling spaces with locally cardinal scale of measurement.\footnote{A preference system is a triplet $\mathcal{A}=[A, R_1 , R_2]$ consisting of a non-empty set $A \neq \emptyset$, a pre-order $R_1 \subseteq A \times A$ on $A$, and a pre-order
 $R_2 \subseteq R_1 \times R_1$ a on $R_1$. Intuitively, the relation $R_1$ captures the available ordinal information, whereas $R_2$ encodes the information's cardinal part.} Each such preference system is then describable by a set of functions $\mathcal{N}_{\mathcal{A}}$, where each element of this set is of the form $\phi: [0,1]^K \to [0,1]$.

 The selection of the optimal unlabeled data would then consequently be based on this same set $\mathcal{N}_{\mathcal{A}}$. To generalize the already mentioned Bayes criterion to this set of utility functions, there are a lot of possibilities (for a compilation of these see in particular~\cite{jsa2018}). We will only briefly discuss here the one among them that does not need to make any additional assumptions and is a consequential generalization of first-order stochastic dominance to our partial cardinal setting. The idea of this generalization is straightforward: If still $\pi$ denotes the prior distribution on the set $\Theta$ of states of nature (= parameters), then now -- instead of choosing unlabeled data that maximize expected utility w.r.t. some fixed utility function --  we exclude all unlabeled data which is expectation-dominated by some other data for all compatible functions $\phi: [0,1]^K \to [0,1]$. More formally, the solution to the decision problem from definition~\ref{def:multi-model-util} with respect to this generalized stochastic dominance criterion is then given by the set $\mathbb{A}_{\mathcal{U}}^{\pi}$ defined by
 $
     \{ a|~ \nexists a': d_{\pi}(a',a)\geq 0 \wedge d_{\pi}(a,a') <0\}
 $, where, for $a_1,a_2 \in \mathcal{U}_{\mathbb{A}}$, we set $d_{\pi}(a_1,a_2)=\inf_{\phi \in \mathcal{N}_{\mathcal{A}}} \left[\mathbb{E}_{\pi}(\phi \circ u (a_1, \cdot))-\mathbb{E}_{\pi}(\phi \circ u (a_2, \cdot))\right].$ 

Importantly, note that all elements remaining in the above set are \textit{incomparable} with respect to the considered criterion of optimality, that is, each of them is an equally plausible candidate for the best next unlabeled data point. In case domain-specific knowledge induces a preference for some of the models under consideration that can be expressed by weights, one might as well simply scalarize the single likelihoods as follows.

\begin{definition}[Weighted Sum of Likelihoods]
\label{def:weighted-sum-util}
The utility function $u \colon \mathbb{A}_{\mathcal{U}} \times \tilde \Theta \to \mathbb{R};$
\begin{align*}
  ((x_i, \mathcal{Y})_i, \theta) &\mapsto  \sum_k w_k \cdot \ell(i, k),
  \end{align*}
with weights $w_k \in (0,1)$, $k \in \{1, \dots, K\}$ summing up to 1, shall be called weighted sum of likelihoods. 
\end{definition}

The respective Bayes criterion (cf. section~\ref{sec:dec-problem}) with multi-model likelihood utility is a weighted sum of posterior predictives of pseudo-labeled data (cf. ibid.). This fact follows directly from theorem 2 in~\cite{rodemann2023-bpls} as well as from the additivity and homogeneity of the expected value.

 Remarkably, the following should be noted: The Bayes-optimal pseudo-labeled data, i.e.~the optimal solutions of the decision problem for selecting pseudo-labeled data according to the Bayes criterion, are always elements of the set $\mathbb{A}_{\mathcal{U}}^{\pi}$ considered before. This means in particular that the aforementioned generalized stochastic dominance and the Bayes criterion based on multi-model likelihood utility are compatible in the sense that the latter -- independent of the concrete weights -- ensures that no labels excluded by the former are chosen. This suggests the following recommendation for criterion selection in concrete application situations: If no content-motivated way of choosing the weights of the multi-model likelihood utility is available, further analysis should rely on the set $\mathbb{A}_{\mathcal{U}}^{\pi}$ alone. If, on the contrary, there is the possibility to determine the weights informed by the content, the Bayes criterion based on the multi-model likelihood utility provides more precise and -- then also non-arbitrary -- results. We now consider a case where a natural choice of weights via penalization of model complexity is appropriate, namely nested models.

\subsubsection{Nested Models}

Now let the models under consideration $M_1, \dots, M_K$, $K < \infty$, be nested in the sense of $\Theta_1 \subseteq \Theta_2  \subseteq \dots \subseteq \Theta_K$. We can interpret the so-induced hierarchy on the parameter space such that the lower $k \in \{1, \dots, K\}$, the simpler the hypothesis space. Aiming at regularization of PLS, we could penalize the respective likelihood utilities of more complex models. In definition~\ref{def:weighted-sum-util}, this could imply e.g. setting $w_k = \frac{\dim(\Theta_k)}{\dim(\Theta_K)}$ for all $k \in \{1, \dots, K\}$.\footnote{One could also weight the likelihoods $\ell(i,1), \dots, \ell(i,K)$ directly in the general case of the multiobjective utility from definition~\ref{def:multi-model-util}.  }  

However, we will opt for a safer approach that guarantees plausibility of at least some pre-specified level $\tau$ under all models $M_1, \dots, M_K$. We therefore draw on the common practice of thresholding selection criteria when selecting pseudo-labeled data in self-training. That is, not only one data point with highest selection function value but all above a threshold are to be selected. We propose to extent this to an intersection of thresholds resulting from likelihood utilities from different models.  

\begin{definition}[Bayesian Multi-Model Threshold Criterion]
As in definition~\ref{def:pseud-lik}, let $(x_i, \mathcal{Y})_i$ be any decision (selection) from $\mathbb{A}_{\mathcal{U}}$. We assign utility to each $(x_i, \mathcal{Y})_i$ given $\mathcal{D}$ and pseudo-labels $\hat{y} \in \mathcal{Y}$ by the multi-model likelihood utility function from definition~\ref{def:multi-model-util}. Now consider the following thresholding Bayes criterion $\Phi_{\tau, \xi, \pi} \colon \mathbb{A}_{\mathcal{U}} \to \mathbb{R};$

\begin{align*}
a \mapsto \Phi_{\tau, \xi, \pi}(a) =  \begin{cases}
  0, & \exists k: \mathbb{E}_\pi(\ell(i, k)) < \tau \\
  0.5, & \forall k: \tau < \mathbb{E}_\pi(\ell(i, k)) < \xi, \\
1, & else. \end{cases}
\end{align*} 
 again with $\ell(i, k) = p(i \mid f_k(\theta), M_k)$, $k \in \{1, \dots, K\}$, and $\xi > \tau$ some pre-specified thresholds. 
\end{definition}

Note that this corresponds to thresholding all pseudo posterior predictive, respectively, see section~\ref{sec:dec-problem}. For parametric models like additive regressions with $K = dim(\Theta)$ we can exploit the hierarchy among models induced by the number of parameters $K$. Before running the procedure (see algorithm~\ref{alg:main}), we start by thresholding pseudo-labeled data based on the full model ($K$ covariates). We refit with lower threshold if no data is selected. If a positive number of data is selected, we kick off the algorithm. We begin decreasing $k$ in a step-wise manner and terminate the process if none of the pseudo-labeled data that were selected in all previous rounds makes it past the threshold. The pseudo-code in algorithm~\ref{alg:main} describes the procedure. 

\begin{algorithm}[H]
\caption{Reversed Occam's Razor}
\label{alg:main}

\KwData{$\mathcal{D}, \mathcal{U}$, set $\mathcal{S}_{K + 1} = \emptyset$}
\KwResult{$\mathcal{D}$}
\For{$k \in \{K, \dots, 1\} $}{
\For{$i \in \{1, \dots, \lvert \mathcal{U} \rvert \}$}{
\textbf{predict} $\mathcal{Y} \ni \hat y_i = \hat y(x_i)$ \\
\textbf{evaluate} $\mathbb{E}_\pi(\ell(i, k))$ \\ 

}
\textbf{select} $ \mathcal{U} \supseteq \mathcal{S}_k = \{ (x_i, \hat{y}_i)_i \mid \Phi_{\tau, \xi, \pi}(a) = 1, a \sim i \} $ \\
\textbf{if} $\mathcal{S}_k \cap \mathcal{S}_{k+1} \neq \emptyset$ : update $\mathcal{D} = \mathcal{D} \cup \mathcal{S}_l$  \\ 
\textbf{else} stop

}
\end{algorithm}

We thus ensure not only $\forall k \in \{1, \dots, K\}:\mathcal{S}_k \neq \emptyset$, but also $\bigcap\limits_{k = 1, \dots, K} \mathcal{S}_k \neq \emptyset.$
That is, among those elements that can be explained equally well by a fixed model, we opt for those that are explained similarly well by simpler models. This can be viewed as reversing Occam’s razor, since we are concerned with selecting data instead of hypotheses. Occam’s time-honored razor advocates selecting the hypothesis with least assumptions among competing hypotheses that have the same explanatory power regarding a single phenomenon. Conversely, we consider multiple phenomena and choose those ones which can still be explained by the simplest hypothesis from a set of competing ones. Occam's razor can be operationalized by Bayesian statistics through the marginal likelihood (or Bayesian evidence), see~\cite{jeffreys1998theory, rasmussen2000occam, mackay2003information, lotfi2022bayesian} for instance. Recall that the Bayesian selection of pseudo-labeled data corresponds to selection with regard to the posterior predictive which is nothing but a marginalized version of the pseudo-labeled data's likelihood.\footnote{Marginalized with regard to the posterior.} Just like in model selection by Bayesian evidence or Bayes factors (i.e., ratios of marginal likelihoods), we are concerned with how well data can be explained by a model. The only difference is that we are interested in comparing (pseudo-)data by how likely it is given a model and not vice versa.




\subsection{Accumulation of Errors: In All Posteriors?}
\label{sec:acc-errors}
The most inherent uncertainty in PLS is caused by the fact that pseudo-labeled data are treated as ground truths in subsequent iterations. 

\begin{definition}[Multi-Label Likelihood as Utility]
\label{def:multi-label-util}
As in definition~\ref{def:pseud-lik}, let $(z, \mathcal{Y})$ be any decision (selection) from $\mathbb{A}_{\mathcal{U}}$. Conversely to definition~\ref{def:pseud-lik}, we now consider not only the predicted pseudo-labels $\hat{y}_i \in \mathcal{Y}$, but also all other hypothetical labels $\tilde y_{i} \in \mathcal{Y} \setminus \{\hat y_i\}$. Denote by $\tilde y_{i,j} \in \mathcal{Y}$ all possible labels for $(x_i, \mathcal{Y})_i$ with $j \in \{1, \dots, J\}$ and $J = \lvert \mathcal{Y} \rvert$. We assign utility to each $(x_i, \mathcal{Y})_i$ by the following utility function $  u \colon \mathbb{A}_{\mathcal{U}} \times \Theta \to \mathbb{R};$
\begin{align*}
  ((z, \mathcal{Y}), \theta) &\mapsto  \sum_{j = 1}^J w_j \cdot p(\mathcal{D} \cup (z, \tilde{y}_{i,j})\mid \theta, M)
  \end{align*}
  with weights $w_j \in (0,1)$ summing up to 1. This utility function shall be called multi-label likelihood.
\end{definition}

Again, the respective Bayes criterion is a weighted sum of posterior predictives of pseudo-labeled data (cf. section~\ref{sec:dec-problem}), because of theorem 2 in~\cite{rodemann2023-bpls} and the additivity and homogeneity of the expected value. A logical choice for the weights $w_j \in (0,1)$ would be the predicted probability of the respective $j$-th label, i.e. $p((z, \mathcal{Y}) = (z, \tilde y_{j,i}))$. This appears quite intuitive. However, while allowing to characterize the unlabeled data points by their plausibilty with hypothetically assigned labels one is still forced to add them with their actually predicted label.

As of now, we loosen this restriction. Notably, definition~\ref{def:pseud-lik} and thus all subsequent deliberations depended on a model $M$ as well as on already predicted labels. We have relaxed the former dependency, while having left the latter untouched. The following remark calls this into question.



\begin{remark}[Sub-Optimal Labels Are Not Redundant]
\label{remark:redundancy psl}
    Consider $u \colon \mathbb{A}_{\mathcal{U}} \times \Theta \to \mathbb{R}$ from definition~\ref{def:pseud-lik} with $\hat y_i = \hat y_{\hat \theta_{ML}}(x_i) $ and $\hat \theta_{ML} = \argmax_\theta p(\mathcal{D} \mid \theta, M)$ the maximum-likelihood estimator. Furthermore, consider $  u \colon \mathbb{A}_{\mathcal{U}} \times \Theta \to \mathbb{R};$
    \begin{align*}
  ((z, \mathcal{Y}), \theta) &\mapsto u((z, \mathcal{Y}), \theta) = p(\mathcal{D} \cup (z, \tilde{y}(z))\mid \theta, M),
    \end{align*} 
    where $z = x_i$ and $\tilde{y}(z) = \tilde y_i= \hat y_{\tilde\theta}(x_i)$ with any sub-optimal $\tilde \theta \in \Theta$ such that $ p(\mathcal{D} \mid \tilde \theta, M) \leq p(\mathcal{D} \mid \hat \theta_{ML}, M)$. It holds that the max-max-action $ a^*_{m} = \max_a \max_\theta u(a, \theta)$ w.r.t. $u$ does generally not have lower utility than the max-max-action $ \tilde a^*_{m} = \max_a \max_\theta u(a, \theta)$ w.r.t. $\tilde u$. To see this, let ${a}^*_{m}$ be the max-max action under $u$ as above. It holds $ a^*_{m} = \max_a \max_\theta (p(\mathcal{D} \cup (z, \hat{y}_i)\mid \theta, M)) = \max_a p(\mathcal{D} \cup (z, \hat{y}_i)\mid \hat \theta_{ML}, M)$. Analogously, $\tilde{a}^*_{m}$ maximizes $p(\mathcal{D} \cup (z, \tilde{y}_i)\mid \hat \theta_{ML}, M)$. As both $(z, \tilde{y}_i)$ and $(z, \hat{y}_i)$ were not considered in ML estimation, we cannot make any statement about the relation of $ u(a_{m}^*)$ to $\tilde u(\tilde a_{m}^*)$. The same holds for the Bayes criterion, as also the posterior of $\theta$ does not include $(z, \tilde{y}_i)$ and $(z, \hat{y}_i)$ either.
\end{remark}

Motivated by this remark, let us now consider the standard utility (definition~\ref{def:pseud-lik}) on a different action space $\tilde{\mathbb{A}}_{\mathcal{U}} = \{(z, y_j) \mid y_j \in \mathcal{Y} \; \text{and} \; \exists \, i \in \{n+1, \dots, m\} : (z, \mathcal{Y}) = (x_i, \mathcal{Y})_i \in \mathcal{U}\}$ and a modified (full) Bayes criterion that accounts for a prior $\rho$ on $\mathcal{Y}$ that weights labels proportional to the predictive distribution from the prediction step before, i.e., $\mathbb{E}_{\rho} (\Phi_{\pi}(a)) = \mathbb{E}_{\rho} \mathbb{E}_{\pi}(u(a,\theta))$.


\begin{proposition}[Full Bayes Equates Weighted Utility]
\label{prop:full-bayes}
    In case of $w_j = \rho(y_j) $ the Bayes criterion under multi-label utility (definition~\ref{def:multi-label-util}) defined on $\tilde{\mathbb{A}}_{\mathcal{U}}$ instead of ${\mathbb{A}}_{\mathcal{U}} $ equals the (full) Bayes criterion $\mathbb{E}_{\rho} (\Phi_{\pi}(a))$ on ${\mathbb{A}}_{\mathcal{U}}$.  
\end{proposition}

\begin{proof}
    $ \mathbb{E}_{\rho} \mathbb{E}_{\pi}(u(a,\theta)) = \int_{\mathcal{Y}} \int_{\Theta} u(a, \theta) \, d \pi(\theta) \, d \rho(y_j) = \int_{\mathcal{Y}} p(\mathcal{D} \cup (z, y_j)\mid\theta, M) \, d \rho(y_j) = \sum_{\mathcal{Y}} p(\mathcal{D} \cup (z, y_j)\mid M) \; \rho(y_j) = \sum_{j} p(\mathcal{D} \cup (z, y_j)\mid M) \; \rho(y_j)$
\end{proof}



\subsection{Covariate Shift}

Selection criteria typically render some unlabeled data more likely to be added than others \cite{rodemann2022not}. In the course of self-training, this can lead to a distributional shift of $X$, often referred to as covariate shift. Depending on the stopping criterion, this covariate shift can be propagated to the final model, potentially harming the model's interpretability by techniques from the realm of interpretable machine learning (IML). 
For instance, regions in the covariate space where data is scarce are detrimental to reliable estimates of partial dependencies~\cite{friedman2001greedy}. Notably, this distributional shift affects all previously discussed selection criteria for PLS. In this subsection, we discuss possible extensions that aim at selecting pseudo-labeled data that are optimal with regard to both the \textit{de facto} selected data $\mathcal{D}$ and a hypothetical \textit{i.i.d.} sample $\mathcal{D}'$ that we generate by drawing pseudo-labeled data randomly. In the spirit of the multi-model likelihood utility (definition~\ref{def:multi-model-util}) and in complete analogy to the previously discussed generalizations, we can define a multi-data likelihood utility, rendering PLS robust with regard to covariate shift. The above discussed decision criteria apply as well. Further note that in this special case of a bi-objective, one might also proceed with an interval-valued utility (loss) function as e.g. in~\cite[section 3.2]{schwaferts2019imprecise}.




\begin{definition}[Multi-Data Likelihood Utility]
We assign utility to each $(x_i, \mathcal{Y})_i$ given $\mathcal{D}$, $\mathcal{D}'$ and the prediction functional $\hat y: \mathcal{X} \to \mathcal{Y}$ by the following bi-objective utility function $  u \colon \mathbb{A}_{\mathcal{U}} \times \Theta \to \mathbb{R}^2;$
\begin{align*}
  ((z, \mathcal{Y}), \theta) &\mapsto (\ell_{\mathcal{D}}(i), \ell_{\mathcal{D}'}(i) )',
  \end{align*}  
with $ \ell_{\mathcal{D}}(i) = p(\mathcal{D} \cup (x_i, \hat{y}_i)\mid \theta, M)$ and $ \ell_{\mathcal{D}'}(i) = p(\mathcal{D}' \cup (x_i, \hat{y}_i)\mid \theta, M)$.

\end{definition}

\color{black}

\section{Updating by $\alpha$-cuts}
\label{sec:alpha-cuts}

All robust extensions of PLS discussed in section~\ref{sec:robust-pls} require some second-level information about the involved uncertainties (e.g., model choice, previous confidence, covariate shift). Aiming at an agnostic and universally robust approach to PLS, we turn to imprecise probabilities~\cite{walley1991statistical, augustin2014introduction} and credal sets~\cite{levi, levi1980enterprise}, more specifically to the fruitful frameworks of convex sets of priors \cite{insua2012robust}, $\Gamma$-maximin \cite{seidenfeld2004contrast} and $\alpha$-cut updating \cite{cattaneo2013likelihood, cattaneo2014continuous}.

\subsection{Updating Credal Sets}

Due to our aforementioned general skepticism regarding the initial model fit $\hat \theta$, we would like to weaken the influence of the likelihood on the posterior in a general way. This can be achieved by means of generalizing Bayesian analysis~\cite{walley1991statistical, insua2012robust, augustin2014introduction}. Again, we can avail ourselves of rich decision theoretical literature dating back to~\cite{ellsberg1961risk, koflerentscheidungen, berger1985statistical}. We will borrow from the theory on Max-E-Min \cite{koflerentscheidungen} or equivalently $\Gamma$-maximin, see for instance~\cite{seidenfeld2004contrast, berger1985statistical,gs1989,t2007, guo2010decision}. To this end, we introduce a convex set of priors $\Pi \subseteq \{\pi(\theta) \mid \pi(\cdot) \, \text{a probabilty measure on } \left(\Theta, \sigma(\Theta) \right) \}$ with $\Theta$ compact as above and $\sigma(\cdot)$ an appropriate $\sigma$-algebra.\footnote{The priors in $\Pi$ can reflect uncertainty regarding prior information, but might as well represent priors near ignorance, see e.g.~\cite{benavoli2015prior, mangili2015new, mangili15prior, rodemann2021robust, rodemann-BO-isipta, rodemann2022accounting}}.    
The rough idea now is this: After observing data, we base our selection (action) on the prior from $\Pi$ that corresponds to the lowest posterior from the set of resulting posteriors. In other words, we hedge against the worst-case prior. In a nutshell, we select the pseudo-labeled instance that would have had the highest expected utility (likelihood) if we had specified the prior in such a way that it contradicted the (potentially overfitted) model's likelihood the most. The respective decision criterion would be the $\Gamma$-maximin criterion $\Phi_{\Pi} \colon \mathbb{A}_{\mathcal{U}} \to \mathbb{R}; \,
    a \mapsto \Phi(a) = \underline{\mathbb{E}}_\Pi(u(a,\theta))$ with $\underline{\mathbb{E}}_\Pi(u(a,\theta)) = \inf_{\pi \in \Pi} \mathbb{E}(u(a,\theta))$ the lower expectation, which we assume to be affinely superadditive (thus equating coherent lower previsions) in the following. This will allow us to exploit the $\alpha$-cut updating rule introduced by \cite{cattaneo2014continuous} for lower previsions. The lower expectation corresponds to the posterior predictive with regard to the posterior that results from updating the prior $\pi^*(\cdot) \in \Pi$ that has the lowest value in the maximum-likelihood estimator $\hat \theta_{ML}$. 


Such an approach, however, might be too much of a good thing, since its respective decisions can completely disregard the likelihood, not to mention its high sensitivity towards $\Pi$. Instead, we opt for an updating rule of credal sets leaning on~\cite{cattaneo2013likelihood, cattaneo2014continuous}\footnote{Updating rules of similar nature have already been introduced by~\cite{moral1990updating, moral1992calculating, gilboa1993updating}. Notably,~\cite[p. 46f]{good1983good} introduced the special case of $\alpha = 1$ as \say{type 2 maximum likelihood}, see also \cite[section 3.5.4]{berger1985statistical}.}: Cattaneo's $\alpha$-cut updating rule with $\alpha \in (0,1)$, also referred to as \say{soft revision} \cite{augustin2021comment}. Its rough idea is to only update those priors whose respective marginal likelihood (evidence) is larger or equal than $\alpha$ times the corresponding maximum marginal likelihood. In other words, the priors whose (relative) likelihood is below $\alpha$ are discarded from the set of lower expectations, before updating all prior lower expectations to posterior lower expectations in this set. This implies restricting the set of alle posteriors to 
\begin{equation}  
\label{eq:alpha-cut}
\{\pi \in \Pi \mid m(\pi) \geq \alpha \cdot \max_{\pi} m(\pi) \},
\end{equation}
with $m(\ell, \pi) = \int_\Theta \ell(\theta) \pi(\theta) d \theta$ the marginal likelihood. This way, we can make sure no decision is made in complete disregard of the likelihood, i.e., based on a $\theta$ with tiny likelihood.

What is more, the $\alpha$-cut updating rule allows for a dynamically adaptive selection of pseudo-labelled data. Note that each predicted pseudo-label $\hat y$ comes with a predicted probability $\hat p_{\hat y} \in [0,1]$ for $\hat y$ to be the true label. After selecting $(x_i, \mathcal{Y})_i$ with respective $(x_i, \hat y_i)$, the probability $\hat p_{\hat y}$ represents our belief in the data $\mathcal{D} \cup (x_i, \hat y_i)$ under which the subsequent model's likelihood is specified.\footnote{Not without a dash of impudence, we might as well borrow from frequentist reasoning and interpret $1 - \hat p_{\hat y}$ as frequency of error.} More generally, in iteration $t$ of SSL, our belief in the pseudo-labeled data is $\prod_{t = 1}^T  \hat p_{\hat y, t}$. We thus could update $\Pi$ in iteration $t$ by $\alpha$-cuts such that $\alpha_t = \prod_{t = 1}^T  \hat p_{\hat y, t}$. 
The interpretation of such an adaptive $\alpha$-cut rule is this: The less we trust the pseudo-labeled data, the wider the cuts should be, since we want to make sure not to down-weight a $\theta$ only because our possibly flawed data says so. Vice versa, if we trust the pseudo-labeled data, we can be more restrictive with regard to the cuts. While providing this strong intuition, we could not find any guarantees for an updating rule of this kind so far. Hence, in what follows, we will motivate an updating rule for SSL based on the expected regret of having considered specific predictions from one specific model in PLS.


\subsection{A Regret-Based Updating Rule}


 

The previous deliberations on model selection (section~\ref{sec:model-sel}) and non-redundancy of sub-optimal labels (remark~\ref{remark:redundancy psl}) motivate our modification of the $\alpha$-cut updating rule for PLS: We update $\Pi$ such that our Bayes action has some quantifiable guarantee with regard to a regret (as ratio, see definitions~\ref{def:label-regret} through~\ref{def:tot-regret}) that stems from both the possibly wrong labels and the possibly wrong models.\footnote{Note that reasoning with both sets of priors and model imprecision is reminiscent of~\cite[chapter 8]{walley1991statistical}} Thus, we start by quantifying these two regrets as random variables on $\Theta$, before defining the total regret as a (posterior) expectation of a function of the two regrets. 

\begin{definition}[Label-Induced Regret]
\label{def:label-regret}
    Consider $\tilde y_{i,j} \in \mathcal{Y}$ all possible labels for $(x_i, \mathcal{Y})_i$ with $j \in \{1, \dots, J\}$ and $J = \lvert \mathcal{Y} \rvert$. As in remark~\ref{remark:redundancy psl}, let $\tilde u_j$ be the pseudo-label likelihood $\tilde u_j(\cdot, \cdot)$ (definition~\ref{def:pseud-lik}) with $\tilde y_i = \tilde y_{i,j}$. Furthermore, set $\hat y_{i,h} = \hat y_{\hat \theta}(x_i)$ as actually predicted label, see remark~\ref{remark:redundancy psl}. For any given decision $a^*$ and any $\theta$, the function $    r_l(\cdot,a^*) \colon \Theta \to \mathbb{R};$
    \begin{align*}
    \theta &\mapsto r_l(\theta,a^*) = \sup_j \frac{u_j(\theta, a^*)}{\tilde u_h(\theta, a^*) } 
    \end{align*}    
     is said to be the label-induced regret.
\end{definition}

\begin{definition}[Model-Induced Regret]
   Let $M_1, \dots, M_K$ and $\Theta_1, \dots, \Theta_K$ denote all models under consideration and their parameter spaces, respectively, as well as $\tilde \Theta = \times_{k = 1}^K \Theta_k$ their Cartesian product. As in definition~\ref{def:multi-model-util}, consider as $\ell(i,k) = p(i \mid f_k(\theta), M_K)$ the likelihood utility of selecting $(x_i, \mathcal{Y})_i$ given model $M_k$ with the projection on $\Theta_k$. In analogy to definition~\ref{def:label-regret}, denote by $M_h$ the actually used model. For any decision $a^* \hat{=}\, i^*$ and any $\theta \in \tilde \Theta$, the function $r_m(\cdot,a^*) \colon \tilde \Theta \to \mathbb{R};$  
    \begin{align*}
    \theta &\mapsto r_m(\theta,a^*) = \sup_k \frac{\ell(i^*,k)}{\ell(i^*,h)} 
    \end{align*}    
     is said to be the model-induced regret.
    
\end{definition}

\begin{definition}[Total Prediction Regret in SSL]
\label{def:tot-regret}
Denote by $\tilde u_{j,k}(\theta, a^*)$ the utility of $a^* \hat{=} \, i^*$ with prediction $\tilde y_{i^*,j}$ under model $M_k$. The function $    r(\cdot,a^*) \colon \tilde \Theta \to \mathbb{R}$
    \begin{align*}
    \theta &\mapsto r(\theta,a^*) = \frac{\sup_{j,k} \tilde u_{j,k}(\theta, a^*)}{\tilde u_{h,h}(\theta, a^*)} 
    \end{align*}
shall be be called total (prediction) regret. 
    
\end{definition}

\begin{definition}[Expected Total Regret Functional]
Based on definition~\ref{def:tot-regret}, the expectation functional $ \Theta \times \Pi \to \mathbb{R}; (\theta, \pi) \mapsto \mathbb{E}_{\pi}(r(\theta, a^*))  $ for given $a^* \in \mathbb{A}_{\mathcal{U}}$
with posterior $\pi \in \Pi$ is said to be the expected total regret functional.
\end{definition}


We can now define an $\alpha$-cut updating rule such that the posterior credal set is 


\begin{equation}    \label{credalupdate}
\Pi_{\alpha} = \{\pi \in \Pi \mid m(\ell_{h,h}, \pi) \geq \alpha \cdot \sup_{j,k} m(\ell_{j,k}, \pi) \}.
\end{equation}


Note that this is just a robustified version of the generic $\alpha$-cut updating according to equation~\ref{eq:alpha-cut}, such that it gives us the following guarantee with regard to the expected regret.

\begin{proposition}[Myopic Regret-Guarantee of $\alpha$-Cuts]
    Bayes optimal selections $a^*$ of pseudo-labeled data under the above $\alpha$-cut updating rule have expected total regret $\mathbb{E}_{\pi}(r(\theta, a^*)) \leq \frac{1}{\alpha}$ for any posterior $\pi \in \Pi$. 
\end{proposition}

\begin{proof}
Consider any $\pi \in \Pi_{\alpha}$. It holds $\forall a \in \mathbb{A}_{\mathcal{U}} : m(\ell_{h,h},\pi) \geq \alpha \cdot \sup_{j,k}m(\ell_{j,k},\pi).$ With $m(\ell,\pi)$ the marginal likelihood w.r.t. to $\pi$ we get: $\forall a \in \mathbb{A}_{\mathcal{U}}: \int_\Theta \ell_{h,h}(\theta) \pi(\theta) d \theta \geq \alpha \cdot \sup_{j,k} \int_\Theta \ell_{j,k}(\theta) \pi(\theta) d \theta \implies \forall a \in \mathbb{A}_{\mathcal{U}}: \mathbb{E}_{\pi}(\ell_{h,h}(\theta)) \geq \alpha \cdot \mathbb{E}_{\pi}(\sup_{j,k}\ell_{j,k}(\theta)) \geq \alpha \cdot \sup_{j,k} \mathbb{E}_{\pi}( \ell_{j,k}(\theta))$. In particular for $a^* \in \mathbb{A}_{\mathcal{U}}$ we have $\frac{1}{\alpha} \geq \frac{\sup_{j,k} {\mathbb{E}}_\pi(\tilde u_{j,k}(a^*,\theta))}{{\mathbb{E}}_\pi(\tilde u_{h,h}(a^*,\theta))} \geq \frac{{\mathbb{E}}_\pi(\sup_{{j,k}} \tilde u_{j,k}(a^*,\theta))}{{\mathbb{E}}_\pi(\tilde u_{h,h}(a^*,\theta))} = \mathbb{E}_{\pi}(r(\theta, a^*))$ with $\ell_{j,k}(\theta) = \tilde u_{j,k}(a, \theta)$. 
\end{proof}

The $\alpha$-cut updating rule was motivated as continuous updating rule by~\cite{cattaneo2014continuous}. This continuity still holds tor the regret-based $\alpha$-cut updating, as follows directly from~\cite[theorem 3]{cattaneo2014continuous}.
\subsection{Generalized Stochastic Dominance under IP}
In the case of using the multi-model likelihood utility from definition~\ref{def:multi-model-util} (rather than a weighted-sum of its components) together with a credal-prior $\Pi$, the criterion of generalized stochastic dominance addressed in section~\ref{GC} can also be easily adapted. Instead of the solution set $\mathbb{A}_{\mathcal{U}}^{\pi}$ used under precise $\pi$, here we would move to the solution set $\mathbb{A}_{\mathcal{U}}^{\Pi}$  robustified under the IP model and defined by 

\begin{equation}    
     \{ a| ~\nexists a': D(a',a)\geq 0 \wedge D(a,a') <0\}, 
\end{equation}

 where, for $a_1,a_2 \in \mathcal{U}_{\mathbb{A}}$, we set 
$D(a_1,a_2)=\inf_{\pi \in \Pi}d_{\pi}(a_1,a_2).$ 
 The interpretation of the set $\mathbb{A}_{\mathcal{U}}^{\Pi}$ robustified by $\Pi$ is similar to the interpretation of the set $\mathbb{A}_{\mathcal{U}}^{\pi}$ under precise $\pi$: It contains all pseudo-labeled data $a$ which are not strictly dominated with respect to generalized stochastic dominance by another pseudo-labeled data $a'$ for no matter which prior $\pi \in \Pi$. Put formally, we thus have that $\mathbb{A}_{\mathcal{U}}^{\Pi}= \bigcap_{\pi \in \Pi} \mathbb{A}_{\mathcal{U}}^{\pi}.$  
 Again, similar to the $\alpha$-cuts method, there are ways to reduce the set of non-excludable pseudo-labels by transitioning from sets $\mathcal{N}_{\mathcal{A}}$ and $\Pi$ to (reasonably chosen) subsets $\mathcal{N} \subset\mathcal{N}_{\mathcal{A}}$ and $\Tilde{\Pi} \subset \Pi$ in the definition of the set $\mathbb{A}_{\mathcal{U}}^{\Pi}$. 
 
 Such a reduction of the set might be desirable, since -- depending on the richness of sets $\mathcal{N}_{\mathcal{A}}$ and $\Pi$ -- set $\mathbb{A}_{\mathcal{U}}^{\Pi}$ might contain too many (possibly even all) available options. In the case of the set $\mathcal{N}_{\mathcal{A}}$, a natural way of reduction is discussed in~\cite{jsa2018} and further deepened in~\cite{jbas2022}: Instead of considering all possible representatives $\phi$ of the underlying preference system, here it is proposed to consider only those that evaluate strict comparability in the underlying partial order above some pre-specified threshold $\xi \in [0,1]$. Also for the reduction of the set $\Pi$ a completely natural possibility offers itself: One can simply shrink the set $\Pi$ by transitioning to the set $\tilde{\Pi}=\Pi_{\alpha}$ from equation~(\ref{credalupdate}) for some reasonable value of $\alpha$. Of course, also combinations of both reduction methods can be used.



\section{Application}
\label{sec:application}

Most of the above decision criteria require the computation of the pseudo posterior predictive (PPP) that involves a possibly intractable integral. MCMC sampling is the usual Bayesian way to circumvent such issues. This in turn usually comes at the cost of some computational hurdles. In order to avoid them, we lean on the analytical approximation of the PPP proposed ~\cite[chapter 3]{rodemann2023-bpls}. For the sake of computational feasibility, we further approximate the log-likelihood given $\mathcal{D} \cup (x_i, \hat{y}_i)$ by the log-likelihood given $\mathcal{D}$, obtaining: $p(\mathcal{D} \cup (x_i, \hat{y}_i)\mid \mathcal{D}, M) \approx 2 \ell (\hat \theta_{ML}) - \frac{1}{2} \log \lvert \mathcal{I}(\hat \theta_{ML}) \rvert$ with $I(\hat \theta_{ML})$ the Fisher information-matrix. We use this approximation to implement three of the above proposed extensions of PLS: multi-label utility (def.~\ref{def:multi-label-util}) as both unweighted and weighted (see proposition~\ref{prop:full-bayes}) sum as well as multi-model utility (def.~\ref{def:multi-model-util}). We benchmark semi-supervised logistic regression with these robust PLS criteria against four common PLS criteria (probability score, posterior predictive (Bayes action), likelihood (max-max action) and predictive variance) as well as a supervised baseline. For the latter, we abstain from self-training and only use the labeled data for training. Experiments are run on simulated binomially distributed data and real world data sets from the UCI machine learning repository~\cite{Dua:2019}. Since target classes are fairly balanced in all data sets, we compare the methods w.r.t. to (test) accuracy. We average the test accuracy for all data sets over a number of repeated self-training rounds each with a new random train-test split.  
The results are promising: For simulated data, PLS w.r.t multi-model PPP achieves accuracy gains of up to 15 percentage points.\footnote{For \textbf{all results}, more details on the experiments and reproducible code, please refer to \href{https://github.com/rodemann/reliable-pls}{www.github.com/rodemann/reliable-pls}.}  
\begin{figure}[h]
    \centering
    \includegraphics[width=\columnwidth]{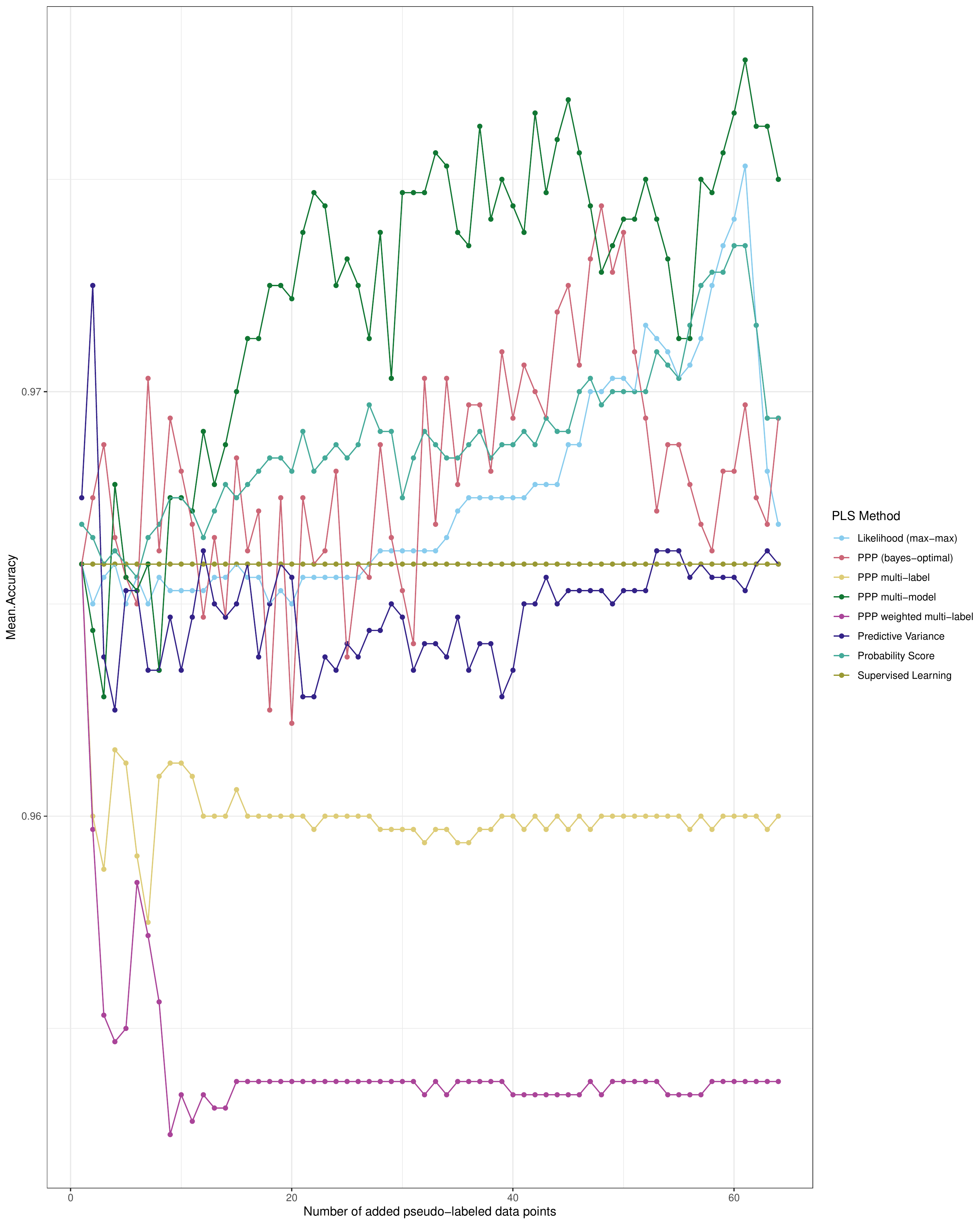}
    \caption{Results from Banknote Data.}
    \label{fig:res-banknote}
\end{figure}

Here, we spotlight the application of our methods on the banknote data~\cite{flury1988multivariate, tortora2014package} that contains measures (diagonal length, bottom margin, length of bill) of 100 genuine and 100 counterfeit Swiss franc banknotes. The learning task at hand is to classify banknotes based on these covariates. Figure~\ref{fig:res-banknote} shows the average accuracy (evaluated on unseen test data, averaged over 40 repetitions) of different PLS methods for $80\%$ unlabeled data. For the multi-model approach, all possible covariate combinations were considered. While multi-model PPP outperforms competing PLS methods, the multi-label extension fails to even beat the supervised baseline. Apparently, it is not worth considering alternative classifications given the initial supervised accuracy is that high ($\sim 0.966$).

\section{Discussion}
\label{sec:discussion}

We have introduced a number of robust extensions of PLS, some of which in turn surfaced avenues for future work. For instance, the accumulated expected errors (section~\ref{sec:acc-errors}) could be used as adaptive learning rate in fractional Bayesian updating~\cite{zhang2006, grunwald2012safe, heide2020safe}. Future work might also focus on implementing and testing the generic generalizations based on $\alpha$-cuts, as introduced in section~\ref{sec:alpha-cuts}. Conclusively, PLS appears to be a promising field for applying existing fruitful frameworks for robust statistical learning such as generalized Bayesian updating using credal sets or more specific multi-model and multi-label robustification. Most of them can potentially be easily transferred to PLS when taking the view on PLS as decision problem. This might not only increase the credibility of the inference by weakening the assumptions, leaning on Manski's \say{law of decreasing credibility}~\cite{manski2003}. It can also, as preliminary evidence suggests, increase predictive performance substantially. In particular, our experiments indicate that considering alternative model specifications as well as non-predicted labels in PLS appears to be a promising and fruitful. Further research is also needed on clarifying interactions among different kind of robustifications (between multi-label and multi-model PLS, for instance).

\acks{%
Georg Schollmeyer would like to thank the LMU mentoring program for support. 
}

\section*{Author Contributions}
Julian Rodemann developed the main idea of robust PLS extensions that account for model selection, accumulation of error and covariate shift. He drafted and wrote the majority of the paper. Julian Rodemann further implemented robust PLS. He also conceived and conducted the experimental analyses. Christoph Jansen contributed the idea of deploying $\alpha$-cut updating rules for robust PLS. Its regret-based adaption was developed by Julian Rodemann. Christoph Jansen contributed several passages on solving robust PLS problems w.r.t. generalized stochastic dominance. Georg Schollmeyer and Christoph Jansen also aided with making technical notations more concise. Thomas Augustin, Georg Schollmeyer, Christoph Jansen and Julian Rodemann further contributed by stimulating discussions and detailed proof-reading.

\bibliography{bib}

\end{document}